\documentclass{article}
\PassOptionsToPackage{sort}{natbib}

\usepackage{iclr2027_conference,times}
\usepackage{hyperref}
\usepackage{url}
\usepackage{amsmath,amssymb,amsthm,mathtools}
\usepackage{booktabs,multirow,array,longtable}\usepackage{graphicx}
\usepackage[caption=false]{subfig}
\usepackage{xspace}
\usepackage{microtype}
\usepackage{enumitem}
\usepackage{xcolor}
\usepackage{tikz}
\usetikzlibrary{arrows.meta,positioning,fit,calc}

\newtheorem{proposition}{Proposition}

\newcommand{\R}{\mathbb{R}}
\newcommand{\E}{\mathbb{E}}
\newcommand{\Prob}{\mathbb{P}}

\newcommand{\NMSE}{\ensuremath{\operatorname{NMSE}}\xspace}

\newcommand{\TrunQuant}{\mbox{TORQUE}\xspace}

\newcommand{\post}{\mathrm{post}}
\newcommand{\pre}{\mathrm{pre}}
\DeclareMathOperator*{\argmin}{arg\,min}

\definecolor{tqteal}{HTML}{007C83}
\definecolor{tqamber}{HTML}{D37A00}
\definecolor{tqblue}{HTML}{3366CC}
\definecolor{tqgray}{HTML}{4B5563}

\title{\TrunQuant: Optimizing What (not) to Quantize\\Before and After Rotation}

\iclrfinalcopy
\setcitestyle{numbers,square,comma}
\author{%
\begin{tabular}[t]{@{}ccc@{}}
\makebox[0.30\textwidth]{\textcolor{black}{\bfseries Ran Ben Basat}} &
\makebox[0.30\textwidth]{\textcolor{black}{\bfseries Michael Mitzenmacher}} &
\makebox[0.30\textwidth]{\textcolor{black}{\bfseries Shay Vargaftik}} \\
\textcolor{black}{\normalfont University College London} &
\textcolor{black}{\normalfont Harvard University} &
\textcolor{black}{\normalfont VMware Research by Broadcom}
\end{tabular}}

\begin{document}

\maketitle
\lhead{} 

\begin{abstract}

Uniform random rotations are an effective preprocessing step for quantization: they make normalized coordinate distributions approximately Gaussian, enabling the use of codebooks optimized offline. We introduce \TrunQuant, a framework that improves on previous quantization works that use random rotations by jointly optimizing how many and which coordinates to preserve at high precision both before and after rotation, under a fixed overall expected bit budget.

Intuitively, before rotation, preserving large input coordinates at high precision can reduce overall error by preventing the rotation from spreading their values across many coordinates. Likewise, after rotation, preserving a small fraction of the largest-magnitude coordinates at high precision allows the remaining values to be quantized more accurately using codebooks optimized offline for the resulting truncated Gaussian distribution.

We derive a quantization error upper bound and prove that top-\(k\) pre-rotation retention minimizes it for each \(k\). This reduces the search over coordinate subsets to an optimization over \(k\), enabling a fast optimizer that uses offline codebooks and parallel parameter selection for practical implementation. We demonstrate an improved tradeoff between reconstruction accuracy and storage cost through numerical evaluation under the Gaussian model and experiments on nearest-neighbor retrieval, KV-cache compression, and activation compression.

\end{abstract}

\section{Introduction}

{\color{black}
Random-rotation-based quantization reduces the memory, communication, and computation costs of machine learning tasks, with applications in distributed training \citep{vargaftik2022eden,suresh2017distributed,vargaftik2021drive,benbasat2024quicfl,dorfman2023docofl,li2024thc,zhao2026mlforml,warraich2025rdma,warraich2025optinic,chen2024trimming,warraich2025optireduce,han2024utility} and inference \citep{chee2023quip,tseng2024quipsharp,ashkboos2024quarot,liu2025spinquant,malinovskii2025higgs}, and approximate nearest-neighbor retrieval \citep{gao2024rabitq,benbasat2026note}.
}
Random rotations regularize the coordinate distribution: under a uniform random (Haar) rotation, every fixed-norm input induces the same spherical distribution, whose fixed-size marginals approach a Gaussian. This allows the use of precomputed codebooks optimized for that distribution with input-independent accuracy guarantees, while randomized Hadamard transforms provide efficient practical implementations \citep{vargaftik2022eden,yang2025raana,panferov2026quartetii}.
Production frameworks also support this approach: vLLM includes a backend with Hadamard rotation, and its Compressor supports rotation-based transforms \citep{vllm2026turboquant,vllm2026transforms}.

We identify two potential sources for improvement in quantization based on random rotations. First, a rotation can obscure useful structure in the original vector. For example, when a few coordinates have particularly large magnitudes (\emph{outliers}), preserving them at higher precision before the rotation can be more effective than spreading their values across the rotated vector. 
Second, the rotated vector can yield outliers that benefit from separate, higher-precision encoding, while the remaining values (the \emph{inliers}) are passed to a lower-bit quantizer. Importantly, this allows the inliers to be quantized over a narrower range, increasing accuracy.
We jointly optimize this outlier selection before and after the rotation to leverage both opportunities under a fixed bit budget.
Outlier retention was previously considered before and after rotation \emph{in isolation} and \emph{without such an optimization}.

\paragraph{Before rotation: preserving structure in the input.}
Special handling of outliers is common even in methods that use no rotation. LLM.int8() isolates outlier feature dimensions for 16-bit multiplication while processing the remaining features in 8 bits \citep{dettmers2022llmint8}; this mechanism is available in Hugging Face's bitsandbytes \citep{bitsandbytes2026llmint8}. SpQR preserves sensitive individual weights at higher precision, and OWQ preserves sensitive weight columns \citep{dettmers2024spqr,lee2024owq}. AWQ also motivates protecting salient weights, but implements that protection through activation-aware channel scaling rather than a retained high-precision subset \citep{lin2023awq}. 
HKVQ considers isolating outliers before quantizing the rest to INT2 format for disaggregated inference~\citep{zhang2025hack,zhang2026hkvq}.

\paragraph{After rotation: quantizing a bounded Gaussian distribution.}
Retention of post-rotation outliers has previously been considered by QUIC-FL \citep{benbasat2024quicfl}. QUIC-FL rotates the input, transmits values outside a central interval separately, and optimizes an unbiased stochastic quantizer for the bounded inlier distribution. There, rotation and outlier retention are complementary: the rotation supplies a predictable distribution, and retaining outliers separately allows finer quantization of the remaining coordinates. Overall, this results in increased accuracy.

The central question is \emph{how many and which values to retain at each stage}. While previous works considered outlier retention at different stages, they did not jointly allocate one vector's bit budget across pre-rotation retention, post-rotation retention, and inlier quantization. The decisions are coupled: retaining values reduces the error but costs value and index bits that could otherwise improve the inliers; pre-rotation retention reduces the norm of the remaining vector before normalization.

We introduce \TrunQuant\ (\emph{Two-stage Outlier-retention Rotated QUantization FramEwork}), a framework that models all these choices as a single optimization problem. Given a total expected bit budget, it determines the pre-rotation and post-rotation outlier retention policies, and the inlier bit-budget and codebook, charging both retained values and their positions. Importantly, the inlier codebooks are designed offline for the resulting induced (truncated Gaussian) distribution.

\TrunQuant\ applies to scalar and vector quantizers. Moreover, \TrunQuant\ preserves unbiasedness when the inlier reconstruction is unbiased. Namely, unbiased estimation of the inliers, exact retention of outliers, and linear recombination through the inverse rotation yield an unbiased estimate.

\paragraph{Contributions.}
\begin{enumerate}[leftmargin=*,itemsep=2pt,topsep=2pt]
  \item \textbf{Joint optimization formulation and framework.}
  We identify outlier retention before and after rotation as coupled decisions and introduce \TrunQuant, a framework that jointly optimizes them under a shared bit budget. Our formulation accounts for the costs of retained values, their indices, and quantizing the remaining coordinates.

  \item \textbf{Theoretical structure and fast online solution.}
  We derive an NMSE upper bound and prove that retaining the largest-magnitude input coordinates before the rotation minimizes this bound. This reduces the search over coordinate subsets to optimizing their size. We combine this with offline \mbox{codebook design and parallel candidate evaluation to enable fast online parameter selection.}

  \item \textbf{Evaluation across quantizers and ML workloads.}
  We numerically quantify worst-case error reductions for scalar and vector quantizers and evaluate on nearest-neighbor retrieval, KV-cache compression, and activation compression. Our results demonstrate that \TrunQuant\ consistently improves the tradeoff between reconstruction accuracy and bit-budget for rotation-based scalar and vector quantization schemes.
\end{enumerate}

\section{The \TrunQuant\ framework}
\label{sec:framework}

\begin{figure}[t]
\centering
\resizebox{.998779\linewidth}{!}{\begin{tikzpicture}[
  font=\fontsize{8}{9.8}\selectfont,
  node distance=7mm and 5mm,
  box/.style={
    draw=tqgray, line width=0.65pt, rounded corners=1.5pt,
    minimum height=11mm, align=center,
    inner xsep=3pt, inner ysep=3pt, fill=white
  },
  stage/.style={box, text width=26mm},
  exact/.style={
    box, draw=tqamber, fill=tqamber!7,
    minimum height=8mm, text width=36mm
  },
  core/.style={box, text width=30mm},
  control/.style={
    box, draw=tqblue, fill=tqblue!6,
    minimum height=7mm, inner ysep=2.5pt, text width=91mm
  },
  arr/.style={-{Latex[length=1.8mm]}, line width=0.7pt, draw=tqgray},
  bypass/.style={arr, draw=tqamber, rounded corners=2pt},
  ctrl/.style={
    -{Latex[length=1.65mm]}, line width=0.6pt,
    draw=tqblue, dashed
  }
]

\node[box, text width=5mm] (x) {{\Large $x$}};
\node[stage, right=of x] (pre)
  {Pre-rotation retention};
\node[stage, text width=22mm, right=of pre] (rot)
  {Normalization \\+ Rotation};
\node[core, right=of rot] (post)
  {Post-rotation retention\\+ \textcolor{blue}{($s$-bit)} Quantization};
\node[box, text width=20mm, minimum height=47mm, anchor=north west]
  (codec) at ([xshift=5mm]post.north east)
  {Representation};

\draw[arr] (x) -- (pre);
\draw[arr] (pre) -- (rot);
\draw[arr] (rot) -- (post);
\draw[arr] (post.east) -- (codec.west |- post.center);

\node[exact, below=3mm of pre] (head)
  {High-precision retained\\values \& indices \textcolor{blue}{(Top-$k$)}};
\node[exact, below=3mm of post] (tail)
  {High-precision retained\\values \& indices \textcolor{blue}{(Threshold $c$)}};

\draw[bypass] (pre.south) -- (head.north);
\draw[bypass] (post.south) -- (tail.north);

\coordinate (headlane) at ($(head.south)+(0,-3mm)$);
\coordinate (taillane) at ($(tail.south)+(0,-1.2mm)$);
\draw[bypass] (head.south)
  -- (headlane) -- (codec.west |- headlane);
\draw[bypass] (tail.south)
  -- (taillane) -- (codec.west |- taillane);

\path let \p1=(pre.north west), \p2=(post.north east) in
  node[control, text width={\x2-\x1-12mm-6.65pt}, anchor=south west]
    (opt) at ([xshift=6mm,yshift=6mm]pre.north west) {Joint optimization selects $k,c,s$};

\coordinate (inputroute) at ($(x.west)+(-4mm,0)$);
\coordinate (budgetport) at ([yshift=1.7mm]opt.west);
\draw[arr, draw=tqblue, rounded corners=2pt]
  (x.west) -- (inputroute) |- ([yshift=-1.5mm]opt.west);
\node[anchor=west, inner sep=0pt, font=\small, text=tqblue]
  (budget) at (inputroute |- budgetport) {Bit budget $b$};
\draw[arr, draw=tqblue, shorten <=1mm]
  (budget.east) -- (budgetport);

\draw[ctrl] (pre.north |- opt.south) -- (pre.north);
\draw[ctrl] (post.north |- opt.south) -- (post.north);

\coordinate (decodelevel) at ($(post.center)+(0,-34mm)$);
\node[core] (mergepost) at (decodelevel)
  {Dequantization \\ + Recombination of\\post-rotation\\ retained coordinates};
\node[stage, text width=22mm] (inverse) at (rot.center |- decodelevel)
  {Inverse rotation\\+ {Denormalization}};
\node[stage] (mergepre) at (pre.center |- decodelevel)
  {Recombination of\\pre-rotation\\ retained coordinates};
\node[box, text width=5mm] (decr) at (x.center |- decodelevel)
  {{\Large $\widehat x$}};

\draw[arr] (codec.west |- mergepost.center) -- (mergepost.east);
\draw[arr] (mergepost.west) -- (inverse.east);
\draw[arr] (inverse.west) -- (mergepre.east);
\draw[arr] (mergepre.west) -- (decr.east);

\coordinate (postlane) at ($(codec.south)+(0,-8mm)$);
\coordinate (postport) at ([xshift=-5mm]mergepost.south);
\draw[bypass]
  ([xshift=-3mm]codec.south)
  -- ([xshift=-3mm]codec.south |- postlane)
  -- node[above=1pt, inner sep=0pt, font=\footnotesize,
          align=center, text=tqamber]
     {Post-rotation high\\precision values \& indices}
     (postport |- postlane)
  -- (postport);

\coordinate (prelane) at ($(codec.south)+(0,-10mm)$);
\draw[bypass]
  ([xshift=3mm]codec.south)
  -- ([xshift=3mm]codec.south |- prelane)
  -- node[pos=0.75, above=3pt, inner sep=0pt, font=\footnotesize,
          align=center, text=tqamber]
     {Pre-rotation high-precision\\values \& indices}
     (mergepre.south |- prelane)
  -- (mergepre.south);

\node[anchor=south west, font=\bfseries\small, text=tqgray]
  at ($(x.north west)+(0,1.5mm)$) {Encoder};
\node[anchor=south west, font=\bfseries\small, text=tqgray]
  at ($(decr.north west)+(0,1.5mm)$) {Decoder};

\end{tikzpicture}
}
\caption{Illustration of the \TrunQuant\ framework. Given an input vector $x\in\mathbb{R}^d$ and a total budget of $db$ bits, the joint optimization selects the number $k$ of largest-magnitude coordinates to retain at high precision before rotation, the retention threshold $c$ after rotation, and the bit budget $s$ with a matching offline precomputed codebook $Q_{c,s}$ for quantizing the remaining entries.}
\label{fig:pipeline}
\end{figure}
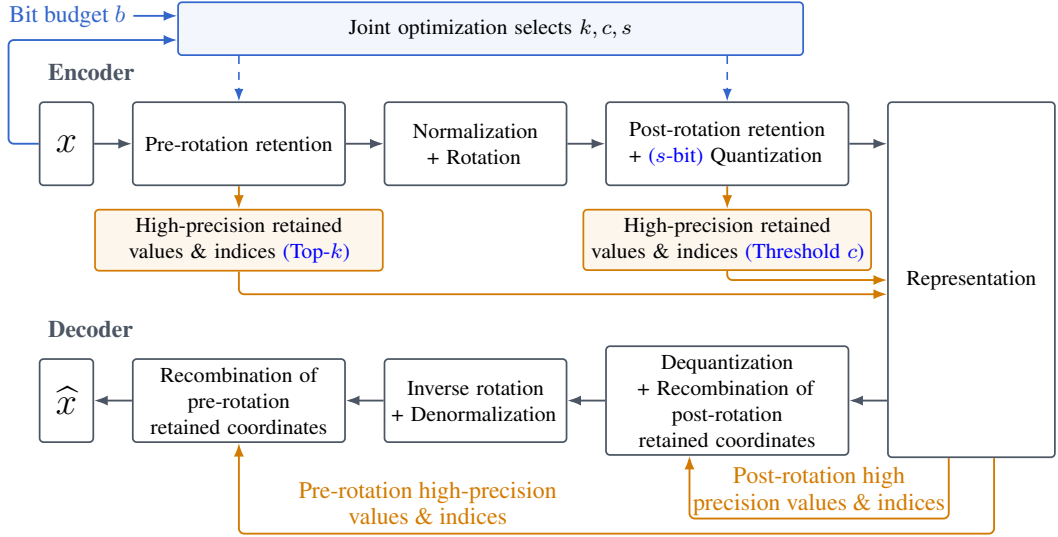

Figure~\ref{fig:pipeline} illustrates \TrunQuant. Given the input and a target bit budget, the joint optimization selects the retained outliers before and after rotation, and the quantizer hyperparameters for the remaining (inlier) entries. We first explain the general workflow and then define the optimization problem.

Specifically, we first describe \TrunQuant with recommended design choices that trade off compression error and computational complexity. Appendix~\ref{app:implementation} discusses alternative implementation choices.

\subsection{Encoding and decoding}
\vspace{-2mm}

Let $k\in \{0,1,\ldots,d\}$, $c\in C \subset (0,\infty]$ and $s\in S \subset [0,\infty)$ be parameters that our joint optimization (\S\ref{sec:torque:subsec:opt_sol}) chooses based on the nonzero input $x \in \mathbb{R}^d$ and the bit-budget $b \ge 0$. Here, $C$ and $S$ are finite sets of supported thresholds and bit allocations.
As detailed in \S\ref{sec:Torque:subsec:Formulation}, for each $(c,s)$, we use an (offline) precomputed optimized codebook $Q_{c,s}$.
\paragraph{Encoding.}
First, the encoder stores the $k$ largest-magnitude input coordinates, along with their positions, and sets those positions in the input to zero. Importantly, the input keeps its original dimension throughout encoding, enabling fast in-place rotation. In practice, since only a small number of coordinates are expected to be stored separately, the additional memory overhead is small.

Choosing the $k$ largest-magnitude coordinates minimizes the error bound for a given $k$: after normalization and a uniform random rotation,\footnote{For theoretical analysis, we assume that the encoder and decoder share a uniform random rotation, also known as the Haar rotation, while randomized Hadamard transforms provide efficient practical implementations \citep{vargaftik2022eden,benbasat2024quicfl,yang2025raana,panferov2026quartetii}. Indeed, recent work has demonstrated that RHTs can approximate Haar rotations with provable guarantees \citep{ben2026quantizing,zilca2026approximating,feng2026provable}.} every input induces the same spherical distribution and hence the same expected quantization error.
Therefore, as we prove in Appendix~\ref{app:topk}, the optimization only needs to search over how many coordinates to retain.

After normalization and rotation, values whose magnitude exceeds the threshold $c$ are stored separately with high precision. For vector quantization, we retain the $q$-sized block if its norm exceeds $c$. The values that are not retained, i.e., the inliers, are then quantized using the precomputed codebook $Q_{c,s}$ which is optimized for the induced truncated Gaussian distribution.

Finally, the encoded representation
includes $c,s$, the retained values and positions from both retention stages, the quantized inlier encoding, and the normalization factor.

\paragraph{Decoding.}
The decoder starts by reconstructing the inliers and reinserting the values retained after rotation. It then applies the inverse rotation and reverses the normalization. Finally, it overrides the entries in the reconstructed vector with the values retained before rotation, using their original positions.
Importantly, if the quantization is unbiased, the decoding preserves unbiasedness.

{\color{black}
\paragraph{Generating the rotation.} As in previous work (e.g.,~\citep{suresh2017distributed,gao2024rabitq}), the encoder and decoder generate the same rotation using a common pseudorandom number generator seed. This means that the rotation need not be stored as part of the representation. 
Previous work also used shared randomness to optimize the quantization~\citep{benbasat2021singlebit,benbasat2025simba,han2026dynamiq}; this is orthogonal to \TrunQuant, and may improve its quantization further.
\par}

\subsection{Formulation}\label{sec:Torque:subsec:Formulation}

\paragraph{Error measure.}
Throughout the paper, we measure reconstruction error using the vector normalized mean squared error (NMSE). For a nonzero input $x\in\R^d$ and its reconstruction $\widehat x$, this is
{\small
\begin{equation}
  \NMSE=\frac{\E\|\widehat x-x\|_2^2}{\|x\|_2^2},
  \label{eq:nmse}
\end{equation}
}
where the expectation is over the rotation and any randomness in the quantizer.

\paragraph{Input normalization and retained coordinates.}

For $k\in\{0,\ldots,d\}$, let $x_{-k}\in\R^d$ be the vector obtained by setting the $k$ largest-magnitude coordinates to zero. We also use $x_k=x-x_{-k}$ to denote the vector with the largest-magnitude coordinates and zeros elsewhere.

Let $
  \rho_k=\frac{\|x_{-k}\|_2^2}{\|x\|_2^2}
$
be the fraction of the squared input norm remaining after retention.
If $\|x_{-k}\|_2=0$, all remaining coordinates are zero. The retained values already specify the entire input, so we skip rotation and quantization. This case has zero error and pays only the cost of pre-rotation retention.

When $\|x_{-k}\|_2>0$, we normalize it as $z_k=x_{-k}\cdot \sqrt{d}/\|x_{-k}\|_2\in \R^d$, which satisfies $\|z_k\|_2^2=d$.

\paragraph{Rotation and quantization.}
Let $R$ be the rotation and write $U=Rz_k$. The rotation makes the expected squared value of each coordinate equal to one.

For scalar quantization, we retain each rotated coordinate whose magnitude exceeds a threshold $c$. Each remaining coordinate lies in $[-c,c]$ and is encoded with a scalar quantizer.

The extension to vector quantization is direct. We partition $U$ into blocks of $q$ coordinates, where $q$ divides $d$, and retain a whole block $U_j$ when $\|U_j\|_2>c $.\footnote{As further optimization, outliers can be defined relative to a codebook $Q\subset\mathbb R^q$, retaining at high precision the blocks farthest from their nearest codeword. The codebook and the remaining inlier region can then be optimized jointly offline by alternating distance-threshold selection with codeword updates as we detail in Appendix~\ref{app:codebook-retention}. We do not employ this optimization since it increases encode time.} Each remaining block is encoded with a vector quantizer. Observe that the case $q=1$ recovers scalar quantization, and $c=\infty$ retains no rotated values in either case. Pre-rotation retention still selects individual coordinates.

\paragraph{Offline codebook construction.}
Before encoding any input, we choose a set $C$ of candidate thresholds and a set $S$ of candidate bit allocations. Here $s\in S$ is the number of bits per inlier coordinate. For each pair $(c,s)\in C\times S$, we design a quantizer $Q_{c,s}$ from the chosen scalar or vector family. The threshold determines the range of values the codebook represents, while $s$ determines how many bits are available to encode them.

Specifically, for a scalar quantizer, we optimize the codebook for a standard Gaussian conditioned to lie in $[-c,c]$. For a vector quantizer, we use a Gaussian block $G\sim N(0,I_q)$ conditioned on $\|G\|_2\le c $. We focus on centroid codebooks for these truncated Gaussian distributions, but our framework also supports any other codebook design.

Importantly, for each pair $(c,s)$, we compute the codebook $Q_{c,s}$ \textit{once}, \textit{offline}, and reuse it across input dimensions.\footnote{Alternatively, one can use adaptive methods to optimize the quantization for the specific rotated vector at the cost of increased complexity \citep{vargaftik2021drive,benbasat2024asq,benbasat2026ecasq}.} We also precompute the expected error of the complete reconstruction, including post-rotation retention, to allow a fast search during optimization.

We first consider scalar quantization. Let $G\sim N(0,1)$. Values with $|G|>c$ are retained exactly, while the remaining values are reconstructed using $Q_{c,s}$. The Gaussian-model error is therefore
\[
  \epsilon(c,s)=\E\left[
    (Q_{c,s}(G)-G)^2 \cdot \mathbf1_{\{|G|\le c\}}
  \right].
\]
Namely, only inliers contribute error.

For vector quantization, we apply the same construction to blocks of $q$ coordinates. We now take $G\sim N(0,I_q)$ and retain blocks whose norm exceeds $c $. The normalized error becomes
\begin{equation}
  \epsilon(c,s)=\frac1q\E\left[
    \|Q_{c,s}(G)-G\|_2^2 \cdot
    \mathbf1_{\{\|G\|_2\le c \}}
  \right].
  \label{eq:dist_nmse}
\end{equation}
We divide by $q$ because $\E\|G\|_2^2=q$. Both expectations include any randomness in the quantizer.

At runtime, the joint optimization selects a codebook together with the 2-stage retention parameters.

\paragraph{Storing retained values and their positions.}
Each retained value requires bits for the value itself and for its position. We use fixed-width indices for the positions. Let $v_{\text{pre}},v_{\text{post}}$ denote the value costs before and after rotation, and let $p_{\text{pre}},p_{\text{post}}$ denote the corresponding index costs. The total costs per retained scalar are
\begin{equation}
  L_{\text{pre}}=v_{\text{pre}}+p_{\text{pre}},\qquad
  L_{\text{post}}=v_{\text{post}}+p_{\text{post}}.
  \label{eq:retained-cost}
\end{equation}

\paragraph{The expected bit budget.}
Let $b$ be the target number of bits per input coordinate, giving a total budget of $db$. We first consider scalar quantization of a nonzero remaining vector. Let $p_d(c)$ be the expected fraction of coordinates retained after rotation. The expected number of stored bits is
\begin{equation}
  \mathcal B(k,c,s)=kL_{\text{pre}}+d\big[p_d(c)L_{\text{post}}+(1-p_d(c))s\big].
  \label{eq:joint-rate}
\end{equation}
The first term counts the values and positions retained before rotation. Inside the brackets, the first term pays for retention after rotation and the second pays for the inlier codes.

When $\rho_k=0$, we instead set $\mathcal B=kL_{\mathrm{pre}}$ and the objective to zero. Any required metadata costs are added before checking feasibility.

For vector quantization, we retain a block $U_j$ when $\|U_j\|_2>c $. We therefore replace the scalar retention probability $p_d(c)=p_{d,1}(c)$ with
$p_{d,q}(c)=\Prob(\|U_j\|_2>c )$.
Each coordinate is retained whenever its block is retained. Thus, \(p_{d,q}(c)\) is also the expected fraction of coordinates retained. Each retained block needs only one index, so its index cost per coordinate, $p_{\text{post}}$, is the number of bits in that index divided by $q$. With these changes, Equation~\eqref{eq:joint-rate} also applies to vector quantization.

The number of values retained after rotation varies with the rotation. Thus, $\mathcal B$ is the expected number of bits used to store retained values, their positions, and inlier codes, averaged over the rotation. The complete representation also includes metadata (e.g., the normalization factor and quantizer identifiers). Since the target budget covers the complete representation, we add the metadata costs to $\mathcal B$ before checking the budget constraint.

\subsection{Joint optimization and solution}\label{sec:torque:subsec:opt_sol}

\paragraph{Reconstruction error.}
The decoder first uses the quantizer to estimate inlier coordinates and overrides the retained positions with their values. Let $\widehat U$ denote this estimate. Inverse rotation gives the approximation of $z_k$:
\[
  \widehat z_k=R^\top\widehat U.
\]
A rotation preserves squared distances. Since $\|z_k\|_2^2=d$, its NMSE is
\[
  \epsilon_d(c,s)
  =\frac1d\E\|\widehat z_k-z_k\|_2^2
  =\frac1d\E\|\widehat U-U\|_2^2.
\]
The distribution of $U$ is the same for every $z_k$. Thus, for the selected quantizer family and dimension $d$, $\epsilon_d(c,s)$ can be precomputed \textit{offline} for each pair $(c,s)$.
The vector $U$ is uniform on the sphere of radius $\sqrt d$, and each coordinate has a scaled, shifted Beta distribution~\citep{vargaftik2021drive}. In practice, we approximate its fixed-size block marginals by Gaussians and use $\epsilon(c,s)$ as an approximation to $\epsilon_d(c,s)$.

To restore the scale, we multiply by $\|x_{-k}\|_2/\sqrt d$. Before overwriting the retained input positions, this gives the intermediate reconstruction
\[
  \widetilde x=x_k+\frac{\|x_{-k}\|_2}{\sqrt d}\widehat z_k,
  \qquad
  x=x_k+\frac{\|x_{-k}\|_2}{\sqrt d}z_k.
\]
Under exact retention,  $x_k$ appears in both expressions and cancels when we subtract them. Therefore,
\[
  \frac{\E\|\widetilde x-x\|_2^2}{\|x\|_2^2}
  =\rho_k\,\epsilon_d(c,s).
\]
At the retained input positions, $z_k$ is zero, so any nonzero reconstructed value there is noise. The decoder replaces this noise with zero before adding the retained values. This can only decrease the squared error. The final reconstruction $\widehat x$ therefore satisfies
\[
  \NMSE
  =\frac{\E\|\widehat x-x\|_2^2}{\|x\|_2^2}
  \le \frac{\E\|\widetilde x-x\|_2^2}{\|x\|_2^2} = \rho_k\,\epsilon_d(c,s).
\]

\paragraph{Selecting the parameters.}
We use the Gaussian approximation of this bound to compare candidates and choose the feasible combination with the smallest modeled error:
\begin{equation}
\boxed{
  (k^*,c^*,s^*)\in
  \argmin_{\substack{k\in\{0,\ldots,d\},\ c\in C,\ s\in S:\ \mathcal B(k,c,s)\le db}}
  \rho_k\,\epsilon(c,s).}
\label{eq:joint-selector}
\end{equation}
The search uses integer values of $k$, and the dimension stays $d$ for every candidate. When the remaining vector is zero, we use the exact reconstruction from retained values described above.

The choices affect each other. Retaining more input coordinates reduces the squared norm of the remaining vector, but leaves fewer bits for quantizing it. Similarly, lowering the post-rotation threshold narrows the inlier distribution, but spends more bits on high-precision values.

We include $k=0$ and $c=\infty$ among the candidates. This allows retention only after rotation, as well as quantization without retention at either stage. Thus, the selected Gaussian-model objective cannot exceed that of any feasible baseline included in the candidate set.

\paragraph{Solution complexity.}
As we prove in Appendix~\ref{app:topk}, pre-rotation retention reduces to choosing the number of retained coordinates. We present two alternative implementation options.

{
For fixed $(c,s)$, $\rho_k$ is nonincreasing in $k$, while $\epsilon(c,s)\ge0$ is independent of $k$. Thus, the largest feasible $k$ minimizes their product. Let $T(c,s)=d[p_d(c)L_{\text{post}}+(1-p_d(c))s]$. Thus, for each $(c,s)$ we obtain
$
  k_{\max}(c,s)=\min\left\{d,\left\lfloor{(db-T(c,s))}/{L_{\text{pre}}}\right\rfloor\right\},
$
provided $k_{\max}(c,s)\ge0$. Otherwise, the pair $c,s$ is infeasible. We sort the squared coordinates once, in $O(d\log d)$ time, and compute their cumulative sums in $O(d)$ time. Evaluating $\rho_{k_{\max}(c,s)}\epsilon(c,s)$ for a given $(c,s)$ then takes $O(1)$ time, giving total parameter-selection time $O(d\log d+|C||S|)$.
Moreover, the evaluation of each {$(c,s)$} candidate is independent and can be parallelized on a GPU. 

For the second implementation, let $\mathcal K=\{k_{\max}(c,s)\mid c\in C, s\in S\}$ be the set of all $k$ values that are of interest.
We run the Multiple Selection algorithm that allows computing the $|\mathcal K|$ quantiles of $|x|$ in $O(d\log |\mathcal K|)$ time.
Accumulating squared-coordinate sums during partitioning gives each required $\|x_{-k}\|_2^2$ and hence $\rho_k$, with ties split to retain exactly $k$ coordinates. 
The total parameter selection time is thus $O(d\log(|C||S|)+|C||S|)$.
}

{ Combining the two options, we obtain a runtime bound of $O(|C||S|+d\cdot \min\{\log d,\log(|C||S|)\}).$}

Small candidate sets are natural in practice. The implementation typically supports only a few bit allocations, keeping $S$ small. We can choose and refine the thresholds in $C$ offline until additional candidates offer little improvement. This balances accuracy against codebook storage and search time, while keeping codebook construction offline and allowing retention to adapt to each input.

For the rotation, the practical choice is to apply a randomized Hadamard transform to $z_k$ at $O(d\log d)$ time, and the decoder can apply its inverse. This option uses the same retention rules and precomputed codebooks. Moreover, using Hadamard compositions or dithering offers $O(d\log d)$ time while providing theoretical guarantees that asymptotically match those of a uniform random rotation \citep{ben2026quantizing,zilca2026approximating,feng2026provable}.

Overall, the complexity remains largely determined by the inlier quantizer and its rotation. \TrunQuant reuses these operations, adding parameter selection and $O(d)$ work for thresholding and recombination. For small, fixed sets $C,S$, this additional work is $O(d\log d)$. Thus, even when the inlier quantizer uses the fast $O(d\log d)$ rotation, \TrunQuant preserves its asymptotic complexity.

\section{Experiments}
\label{sec:experiments}

\textcolor{black}{The application experiments compare the scalar quantizers EDEN~\citep{vargaftik2022eden},}
RaBitQ~\citep{gao2024rabitq,gao2024extendedrabitq}, and
TurboQuant~\citep{zandieh2025turboquant}, with and without \TrunQuant.
Section~\ref{sec:eval-gaussian} also evaluates the vector quantizer
HIGGS~\citep{malinovskii2025higgs} under the Gaussian model.
The main figures use unbiased scalar reconstruction and Gaussian reciprocal
reconstruction for HIGGS. Biased variants appear in
Appendix~\ref{app:experiments}.

Retained values cost $v=16$ bits each and indices cost $p=8$ bits each.
For scalar applications, we divide vectors into consecutive chunks of
at most $256$ coordinates. A vector block of $q$ coordinates shares one index,
so its retention costs $16q+p$ bits and supports at most $256$ blocks
per rotation.

\subsection{Activation and KV-cache compression}
\label{sec:eval-activations}

We evaluate Qwen2.5-0.5B
\citep{qwen2024qwen25} on the WikiText-2 test split
\citep{merity2016pointer}.
Here, we compare the unbiased variants of the above quantizers, while the biased variants are evaluated in Appendix~\ref{app:biased-activation-kv}.

\paragraph{Activations.}
We run the model on eight text segments of $256$ tokens each and
choose eight tokens from each segment. For these $64$ tokens, we
record four activation vectors within each of the model's $24$
decoder blocks: the input to the query/key/value projections, the
input to the attention output projection, the input to the MLP
gate/up projections, and the input to the MLP down projection.
We quantize each recorded vector in consecutive chunks of $256$
coordinates. If $128$ coordinates remain at the end, we quantize them as a separate chunk of size $128$. For \TrunQuant, each chunk uses both~retention stages.

As shown in Figure~\ref{fig:applications-kv}(a), \TrunQuant reduces the \NMSE by $30.8$--$51.5\%$ for EDEN, $31.8$--$86.2\%$
for RaBitQ, and $33.6$--$51.2\%$ for TurboQuant. Interestingly, for non-integer values of $b$, our improvement is larger than for integer ones. This is because it allows \TrunQuant to use an integer number of bits when quantizing inliers and use the residual bits for the retained values. For the raw inlier quantizers, a fractional budget means mixed-precision quantization, leading to a less favorable bit-budget to accuracy tradeoff as shown in EDEN \citep{vargaftik2022eden}.

\paragraph{KV caches.}
\label{sec:eval-kv}
We run the model on $16$ nonoverlapping text segments of $256$ tokens
each. For each segment, we process the first $255$ tokens to construct
the key and value caches in all $24$ decoder blocks. Each decoder block has two key/value heads, each with 64 coordinates, giving 128 key coordinates and 128 value coordinates per token. We quantize each head’s key and value vectors separately, using both retention stages for TORQUE.

As shown in Figure~\ref{fig:applications-kv}(b) and (c), \TrunQuant reduces
mean key and value NMSE at every tested budget. Key NMSE decreases by
$13.9$--$44.1\%$ for EDEN, $16.4$--$94.7\%$ for RaBitQ, and
$19.7$--$44.8\%$ for TurboQuant. The corresponding value NMSE reductions
are $0.8$--$19.6\%$, $1$--$92\%$, and $1.2$--$19.4\%$.

\begin{figure}[t]
\centering
\includegraphics[width=.9995\linewidth]{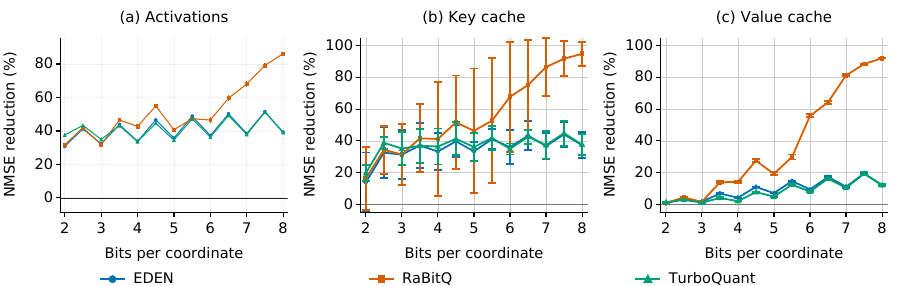}
\caption{Qwen2.5-0.5B \TrunQuant \mbox{improvement:
(a) activations; (b) key caches; (c) value caches.}}
\label{fig:applications-kv}
\end{figure}

\subsection{Nearest-neighbor retrieval}
\label{sec:eval-retrieval}

Figure~\ref{fig:applications-ann} uses the complete GloVe-200~\citep{pennington2014glove} index
of $1{,}183{,}514$ vectors and all $10{,}000$ queries, with cosine
similarity. We reconstruct the quantized database vectors and measure
NMSE and Recall@1: the fraction of queries for which the retrieved
nearest neighbor agrees with the full-precision search.

Figure~\ref{fig:applications-ann}(a) shows the NMSE reduction,
(b) the change in recall, and (c) Recall@1.
Appendix~\ref{app:retrieval} includes the GIST~\citep{jegou2010product} dataset and biased retrieval
comparisons.

At the same bit budget, \TrunQuant lowers NMSE for all three quantizers.
Across the 13 plotted budgets $b=2,2.5,\ldots,8$, the mean paired NMSE
reductions are up to $20.9\%$ for EDEN,
$81.5\%$ for RaBitQ, and $20.3\%$ for TurboQuant, respectively. Recall@1 also improves,
with gains of up to $1.5$ percentage points.

\begin{figure}[t]
\centering
\includegraphics[width=.995\linewidth]{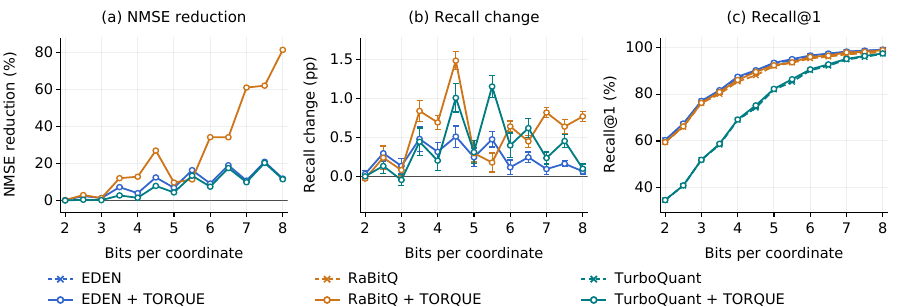}
\caption{GloVe retrieval: (a) NMSE reduction, (b) change in recall, and
(c) Recall@1.}
\label{fig:applications-ann}
\end{figure}

\begin{figure}[!ht]
  \centering
  \subfloat[Gaussian]{\includegraphics[width=.995\textwidth]{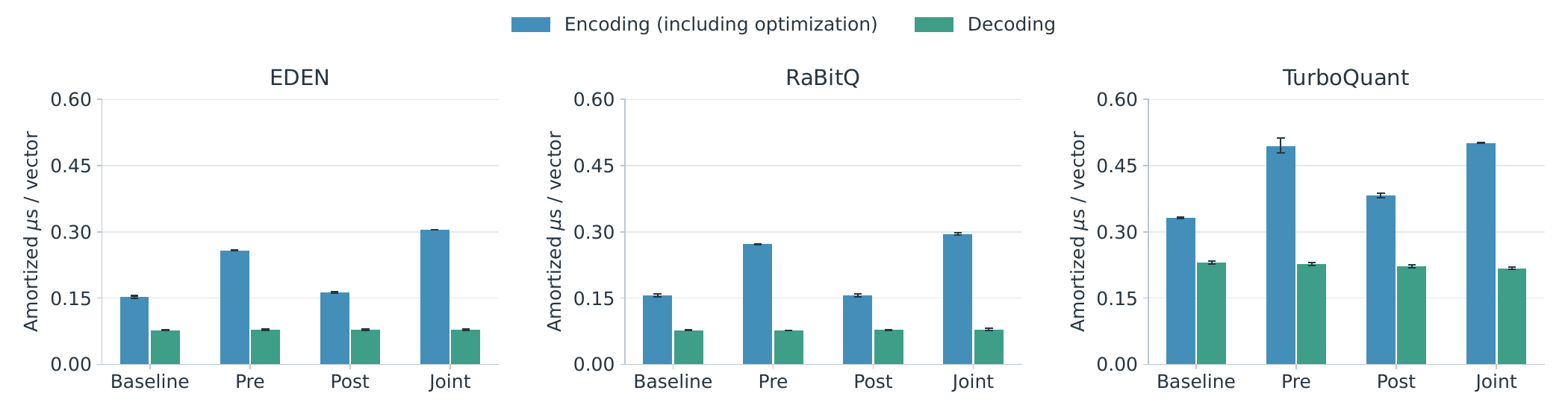}}
    \\\subfloat[Student-$t_3$]{\includegraphics[width=.9949\textwidth]{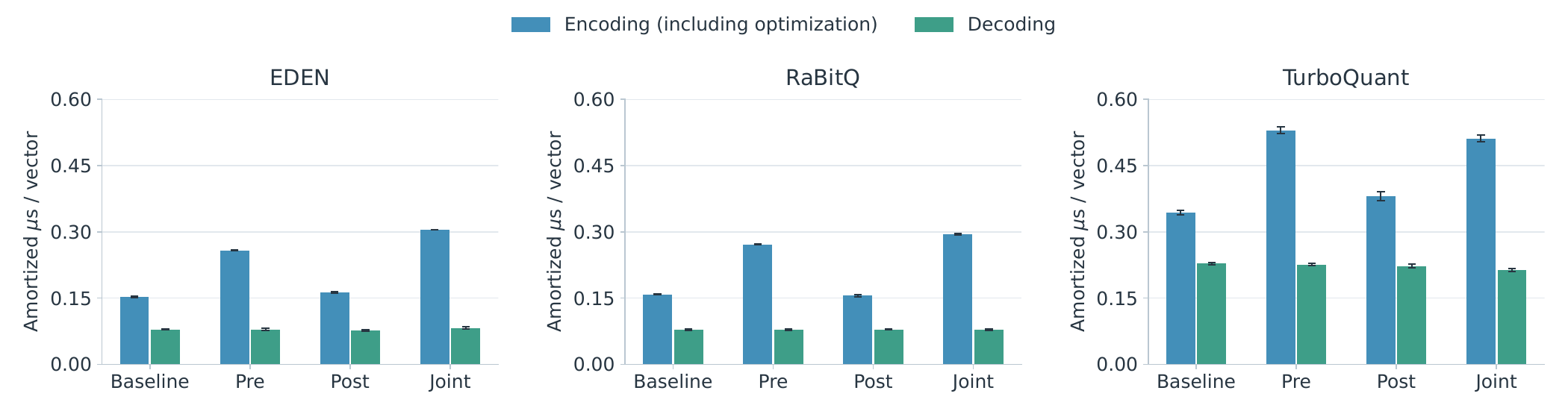}}
  \caption{CUDA encoding and decoding time on an NVIDIA RTX 5090
    for $d=256$, batch size 128, and a target coordinate budget
    of 4.5 bits.
Pre, Post, and Joint
    denote TORQUE retention before rotation, after rotation,
    and at both stages,
    respectively.
}
  \label{fig:cuda-runtime}
\end{figure}

\subsection{Encoding and decoding time}
\label{app:runtime}

Figure~\ref{fig:cuda-runtime} shows encoding and decoding times for
normalized Gaussian and Student-$t_3$ inputs on an NVIDIA RTX 5090.
Encoding includes optimization. Our C++17/CUDA implementation
uses randomized Hadamard rotations and fused kernels that keep
intermediate values in shared memory.

We compile with NVCC 13.3 for \texttt{sm\_120}, MSVC 19.44, and
\texttt{-O3 --use\_fast\_math}. Launch configurations are tuned
separately for each variant on a held-out seed and fixed before
measurement. We average five seeds and report $95\%$ confidence intervals.

At $d=256$, batch size $128$, and a $4.5$-bit coordinate budget,
encoding time relative to the corresponding baseline increases by
$48.9$--$73.7\%$ with pre-rotation retention and by
$48.4$--$99.7\%$ with joint retention, across the three quantizers
and both input distributions.
For post-rotation retention, the encoding-time increase is modest and always under $15.7\%$.
Across all three retention variants,
decoding time is barely affected.

Additional results for
native CPU and Apple GPU implementations appear in
Appendix~\ref{app:local-runtime}.

\subsection{Gaussian-model error and retention choices}
\label{sec:eval-gaussian}

Figure~\ref{fig:skew-gaussian}(a) isolates the benefit of retention
after rotation under the Gaussian model. Since the post-rotation distribution is independent of the input, this corresponds to the \textit{worst-case }attainable \NMSE, up to an additive factor that decays with the dimension~\cite{vargaftik2022eden}.
We compare each quantizer
with and without retention. The scalar curves use EDEN, RaBitQ,
and TurboQuant on standard Gaussian coordinates. The vector curves use
HIGGS on blocks $G\sim N(0,I_q)$ of $q=2$ or $q=4$ coordinates.
HIGGS maps each block to its nearest codeword in Euclidean distance.
For each $c,s,q$ combination, we fit a codebook offline for the induced truncated Gaussian distribution.

At each budget, we select the threshold and bit allocation with the smallest estimated error among the feasible candidates, including
no retention ($c=\infty$).

At the same expected bit budget, \TrunQuant reduces scalar NMSE by
up to $25.69\%$ for EDEN, $53.73\%$ for RaBitQ, and $24.65\%$
for TurboQuant. For HIGGS, the plotted NMSE reductions reach $11.12\%$ for $q=2$
and $3.76\%$ for $q=4$.
Again, we see the benefit of using slightly-more-than-integral values for $b$, which allows \TrunQuant  to use the fractional leftover for the outliers.

Figure~\ref{fig:skew-gaussian}(b) examines the benefit of retaining
large input coordinates before rotation based on the skewness of the input distribution. We compare EDEN with
retention before, after, or at both stages on signed-lognormal and
Student-$t$ inputs. Signed-lognormal coordinates have independent
random signs and log-magnitudes drawn from $N(0,\theta^2)$.
Student-$t$ coordinates have $\nu$ degrees of freedom, plotted as
$h=1/\nu$. Increasing $\theta^2$ or $h$ makes large coordinates
more pronounced. For each of five seeds, we draw $256$ vectors of
dimension $256$ and normalize them to unit norm. All retention choices
share the same inputs and randomized Hadamard rotations.
The budget is eight bits per coordinate.

Relative
to EDEN, \TrunQuant reduces NMSE by up to $99.68\%$. The improvement's source depends on input skewness: when the data is heavy-tailed (low $\theta^2$ or $h$), \TrunQuant's improvement stems from post-rotation retention. In contrast, when the data is skewed, the pre-rotation retention does most of the heavy lifting.
Together, these results illustrate how the two
stages can complement each other.

\begin{figure}[t]
\centering
\begin{minipage}[t]{.49\linewidth}
\centering
\includegraphics[width=\linewidth]{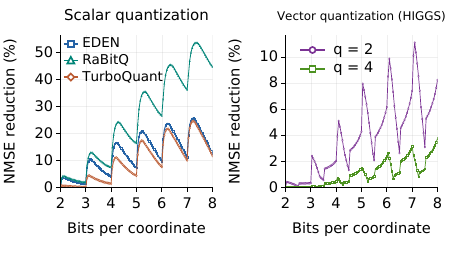}
\par\smallskip
{\small (a) Gaussian-model: scalar and vector quantization.\par}
\end{minipage}\hfill
\begin{minipage}[t]{.49\linewidth}
\centering
\includegraphics[width=\linewidth]{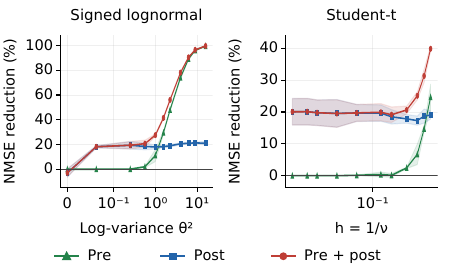}
\par\smallskip
{\small (b) Improvement breakdown compared to EDEN.
\par}
\end{minipage}
\caption{NMSE reductions from applying \TrunQuant.}
\label{fig:skew-gaussian}
\vspace{-1mm}
\end{figure}

\section{Limitations and conclusion}
We presented \TrunQuant, a framework that jointly optimizes outlier retention before and after rotation under a shared bit budget. Our evaluation shows that \TrunQuant improves the reconstruction accuracy of state-of-the-art scalar and vector quantizers across different input distributions and applications.

The main computational cost is the additional encoding time required for online optimization. This relative overhead decreases as the vector dimension grows, while decoding time remains close to that of the underlying quantizer. The encoding cost is less important in applications where data is quantized once and decoded many times, making decoding latency the main concern. Examples include KV caches for reusable system prompts and RAG contexts, \mbox{as well as post-training weight quantization.}

An interesting direction for future work is to reduce the optimization cost by exploiting slowly changing input distributions (e.g., gradients over different training rounds) or sampling a small number of chunks out of a large vector we wish to quantize for optimizing the parameters.

\bibliographystyle{unsrtnat}
\bibliography{references}

\appendix

\section{Optimal retention before rotation}
\label{app:topk}

We justify the top-$k$ rule used in Section~\ref{sec:framework}. The argument uses the rotation and exact-retention assumptions stated there. We first fix the number of retained coordinates and compare the possible sets of positions. Each set must have the same storage cost and allow the same quantizer parameters.

\paragraph{Error for a chosen set of positions.}
Let $x\ne0$ be the input, let $\mathcal I\subseteq\{1,\ldots,d\}$ contain $k$ retained positions, and let $h_{\mathcal I}$ contain the input values at those positions and zeros elsewhere. The fraction of the squared input norm left in the remaining vector is
\[
  \rho_{\mathcal I}=\frac{\|x-h_{\mathcal I}\|_2^2}{\|x\|_2^2}
  =1-\frac{\sum_{i\in\mathcal I}x_i^2}{\|x\|_2^2}.
\]
When $\rho_{\mathcal I}>0$, we normalize the remaining vector as
\[
  z_{\mathcal I}=\frac{\sqrt d}{\|x-h_{\mathcal I}\|_2}(x-h_{\mathcal I}),
  \qquad \|z_{\mathcal I}\|_2^2=d.
\]
For fixed parameters $(c,s)$, let $\widehat z_{\mathcal I}$ be its reconstruction. The rotation makes its expected normalized squared error independent of the direction of $z_{\mathcal I}$, so
\[
  \frac1d\E\|\widehat z_{\mathcal I}-z_{\mathcal I}\|_2^2=\epsilon_d(c,s).
\]
Restoring the scale and adding the retained values gives the intermediate reconstruction
\[
  \widetilde x=h_{\mathcal I}+\frac{\|x-h_{\mathcal I}\|_2}{\sqrt d}\widehat z_{\mathcal I}.
\]
The input has the same expression with $z_{\mathcal I}$ in place of $\widehat z_{\mathcal I}$, so the exactly retained part cancels when we subtract them. The decoder then restores the retained positions to their exact input values, giving $\widehat x$. This replaces the error at those positions by zero and can only decrease the squared error. Therefore,
\begin{equation}
  \NMSE
  =\frac{\E\|\widehat x-x\|_2^2}{\|x\|_2^2}
  \le\frac{\E\|\widetilde x-x\|_2^2}{\|x\|_2^2}
  =\rho_{\mathcal I}\,\epsilon_d(c,s).
  \label{eq:head-error-factorization}
\end{equation}
If $\rho_{\mathcal I}=0$, the retained values already specify $x$ and the error is zero.
If any set of $k$ positions has this property, the top-$k$ set also does and achieves zero error at the same retention cost. The remaining argument concerns nonzero remaining vectors.

\paragraph{Choosing the retained positions.}
For fixed $k$, retention costs $kL_{\text{pre}}$ bits. The dimension, post-rotation position cost, and remaining budget are unchanged when we replace one set of retained positions with another. Thus, the feasible pairs $(c,s)$ are the same. For each pair, only $\rho_{\mathcal I}$ in the bound in Equation~\eqref{eq:head-error-factorization} depends on the chosen positions.

\begin{proposition}[Optimal retention for the error bound]
\label{prop:topk}
Assume exact retention and a nonnegative, direction-independent expected normalized squared error $\epsilon_d(c,s)$ for the normalized remaining vector. Suppose all sets of $k$ retained positions have the same storage cost and feasible quantizer parameters. Then retaining the $k$ largest-magnitude input coordinates minimizes the bound
\[
  \left(1-\frac{\sum_{i\in\mathcal I}x_i^2}{\|x\|_2^2}\right)\epsilon_d(c,s)
\]
over all sets $\mathcal I$ of size $k$, for every feasible pair $(c,s)$. Consequently, the search over $k,c,s$ in Equation~\eqref{eq:joint-selector} needs no additional search over position sets.
\end{proposition}

\begin{proof}
Suppose $i\in\mathcal I$ and $j\notin\mathcal I$ satisfy $x_i^2<x_j^2$. Replacing $i$ with $j$ reduces $\rho_{\mathcal I}$ by $(x_j^2-x_i^2)/\|x\|_2^2$. Since $\epsilon_d(c,s)\ge0$, this replacement cannot increase the bound. It also leaves the storage cost and feasible parameters unchanged.

Repeating these replacements gives the $k$ largest squared coordinates. This set minimizes the bound for every feasible pair $(c,s)$, and therefore also after selecting the best such pair.
\end{proof}

Since $k=0$ is included in the joint search, the minimum objective therefore cannot exceed that of the best feasible candidate using retention only after rotation. The search may select $k=0$ when retaining input coordinates offers no improvement.

\paragraph{Other dimensions and error estimates.}
The same argument applies when retained coordinates are removed before rotation. For fixed $k$, the dimension $m=d-k$ is the same for every position set. We normalize the remaining vector to squared norm $m$ and replace $\epsilon_d(c,s)$ by $\epsilon_m(c,s)$. The same top-$k$ conclusion follows. For vector quantization, this option requires $q$ to divide $m$; when no nonzero coordinates remain, rotation and quantization are skipped.

\section{Alternative implementation choices}
\label{app:implementation}

The formulation in Section~\ref{sec:framework} keeps all $d$ coordinates and uses fixed-width indices. Every choice of $k$ can therefore use the same rotation implementation and precomputed tables of errors and expected bit costs. The retained positions are also simple to store and recover. Together with truncated Gaussian codebooks constructed offline, these choices offer a useful compromise between storage efficiency and a fast practical implementation. Below, we discuss other implementation options, not used in our evaluation, that may offer a better accuracy to representation size trade-off at the cost of increased complexity and storage.

\paragraph{Dimension and reconstruction.}
One option is to remove the retained input coordinates and collect the remaining coordinates into a smaller vector. Write $m$ for the dimension entering the rotation: the main formulation uses $m=d$, while this option gives $m=d-k$. For a nonzero remaining vector, we normalize the smaller vector to obtain $z_k\in\R^m$ with $\|z_k\|_2^2=m$, and write $U=Rz_k$. The decoder rescales by $\|x_{-k}\|_2/\sqrt m$ and reinserts the retained input values at their original positions. The expected bit cost is then
\[
  \mathcal B(k,c,s)=kL_{\text{pre}}+m\big[p_m(c)L_{\text{post}}+(1-p_m(c))s\big]
\]
and the objective becomes $\rho_k\epsilon_m(c,s),$
where $p_m(c)$ and $\epsilon_m(c,s)$ are the expected retained fraction and NMSE at dimension $m$. We substitute them into Equation~\eqref{eq:joint-selector}. For vector quantization, replace $p_m(c)=p_{m,1}(c)$ by $p_{m,q}(c)=\Prob(\|U_j\|_2>c )$ and consider only choices for which $q$ divides $m$. This option processes $d-k$ coordinates, but the rotation size and finite-dimensional tables now depend on $k$.

Since $\epsilon_m(c,s)$ depends on $k$ when $m=d-k$, the largest feasible count need not minimize $\rho_k\epsilon_m(c,s)$. The one-count-per-pair reduction in Section~\ref{sec:torque:subsec:opt_sol} therefore does not automatically apply. Enumerating the admissible counts with precomputed cost and error tables takes $O(d\log d+d|C||S|)$ time.

When keeping $m=d$, the decoder sets the retained positions of $\widehat z_k$ to zero before adding the retained values, as described in Section~\ref{sec:torque:subsec:opt_sol}. This removes their reconstruction noise, so $\rho_k\epsilon_d(c,s)$ is an upper bound on the final NMSE.

\paragraph{Value precision and index widths.}
For scalar retention with fixed-width indices, the position costs are $p_{\text{pre}}=\lceil\log_2 d\rceil$ and $p_{\text{post}}=\lceil\log_2 m\rceil$. With $m=d$, a simple choice is to use the same precision and index width at both stages: $v=v_{\text{pre}}=v_{\text{post}}$ and $p=p_{\text{pre}}=p_{\text{post}}=\lceil\log_2 d\rceil$. The two stages may also use different value precisions, with their costs included in $L_{\text{pre}}$ and $L_{\text{post}}$. For blocks of $q$ coordinates, one index identifies a whole block. The post-rotation cost per retained coordinate is then $v_{\text{post}}+\lceil\log_2(m/q)\rceil/q$.

\paragraph{Position encodings.}
When few input coordinates are retained, we can store differences between consecutive positions or encode the whole subset. These choices can use $p\approx\log_2(d/k)$ bits per position on average, plus overhead.

We can also store retained values in their original order and record where to reinsert them among the $d-k$ remaining entries. Each gap index uses $\lceil\log_2(d-k+1)\rceil$ bits, approximately $\log_2(d-k)$ when $d-k$ is large.

These choices can save index bits but require more coding work. The top-$k$ result in Appendix~\ref{app:topk} requires equal-size position sets to have the same storage cost and feasible quantizer parameters. At fixed dimension and fixed $(c,s)$, a cost that depends only on $k$ preserves the largest-feasible-count rule, but requires solving feasibility for that cost instead of using the fixed-width formula. If the cost depends on the retained positions themselves, the top-$k$ guarantee need not hold.

\paragraph{Encoding block positions jointly.}
Suppose $r$ of the $m/q$ blocks are retained. There are $\binom{m/q}{r}$ possible sets of retained positions.
Enumerative coding stores the index of the $r$-sized subset, which requires $\lceil\log_2\binom{m/q}{r}\rceil$ bits, the minimum fixed-length cost when $r$ is known. The decoder must also know $r$, which can be inferred from the representation or stored using $O(\log m)$ additional bits.

Under the Gaussian model (\S\ref{sec:Torque:subsec:Formulation}), a block is retained with probability $\pi=\Prob(\|G\|_2>c )$. Joint encoding uses roughly $H(\pi)$ bits per block for its position, or $H(\pi)/q$ bits per coordinate. Retained \mbox{coordinates use $v$ bits each, while inlier coordinates use $s$ bits each. Adding these costs gives}
\begin{equation}
 b_G(c,s)=\frac{H(\pi)}q+v\pi+(1-\pi)s,
 \qquad H(\pi)=-\pi\log_2\pi-(1-\pi)\log_2(1-\pi).
 \label{eq:gaussian-rate}
\end{equation}
Here, $H$ is the binary entropy function, with $H(0)=H(1)=0$. The three terms account for positions, retained values, and inlier codes.

For fixed $q,c$, the expected cost per coordinate is at most $b_G+O(\log m/m)$. This bound does not require independent blocks. Including retention before rotation, the expected total cost is therefore at most
$
  kL_{\text{pre}}+m b_G(c,s)+O(\log m).
$

\paragraph{Dimension-dependent codebooks.}
We use codebooks optimized for a truncated Gaussian distribution, making their design independent of the input dimension.

An alternative is to optimize codebooks for the exact distribution induced by the rotation: a scaled, shifted Beta distribution for scalar quantization, or the joint spherical marginal for vector quantization. For each threshold $c$, we condition this distribution on the coordinates or blocks that are not retained, then design one codebook for each bit allocation $s$. At a fixed rotation dimension $d$, this gives $|C||S|$ codebooks, shared by all pre-rotation choices of $k$. If we remove the retained input coordinates, the dimension becomes $d-k$. We then repeat this construction for each supported nonzero dimension, giving up to $d|C||S|$ codebooks.

Another option is to retain a fixed number of the largest-magnitude coordinates after rotation instead of using a threshold. This fixes the retention count and, with fixed-width encoding, makes the bit cost deterministic. However, the inlier distribution now depends on both the rotation dimension and the retention count. We therefore design one codebook for each supported retention count and bit allocation at each dimension. Supporting all counts gives up to $d|S|$ codebooks at a fixed dimension, or $O(d^2|S|)$ if the dimension also varies as $d-k$.

\section{Additional application results}
\label{app:experiments}

We extend the application experiments in Section~\ref{sec:experiments}
to biased reconstruction and GIST retrieval.

\subsection{Biased Activation and KV-cache compression}
\label{app:biased-activation-kv}
Figure~\ref{fig:applications-kv-biased} complements
Figure~\ref{fig:applications-kv} with EDEN-B, RaBitQ-B, and
TurboQuant-MSE, the biased variants of these quantizers. We evaluate
the activations using the setup in Section~\ref{sec:eval-activations}.
For KV caches, we use $256$ contexts and retention only after rotation.

\begin{figure}[!htbp]
\centering
\includegraphics[width=\linewidth]{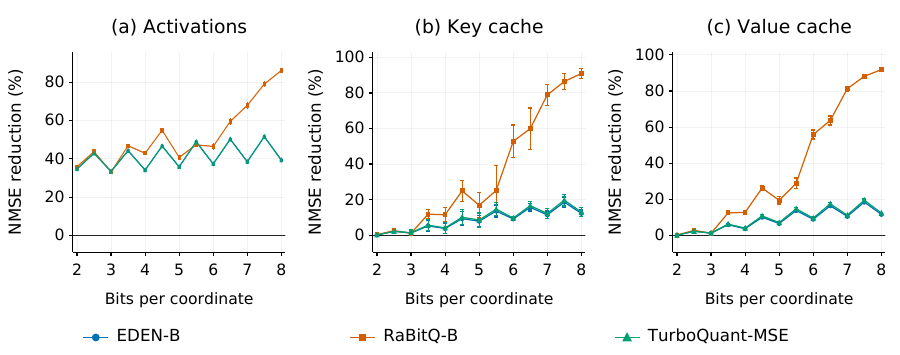}
\caption{Biased reconstruction: NMSE reductions for activations, key
caches, and value caches. The KV-cache results use retention only after rotation.}
\label{fig:applications-kv-biased}
\end{figure}

\subsection{Additional retrieval results}
\label{app:retrieval}

\paragraph{GIST retrieval.}
We extend the scalar comparison in Section~\ref{sec:eval-retrieval}
to squared Euclidean search on GIST-960~\citep{jegou2010product}, using its complete index of
one million vectors and all $1{,}000$ queries.

For unbiased GIST reconstruction, we store the original squared norm
and include its cost in the metadata. For a query $y$ and an index
vector $x$, the distance estimate is
\[
  \|y\|_2^2+\|x\|_2^2-2\langle y,\widehat x\rangle.
\]
It uses the stored $\|x\|_2^2$, rather than the squared norm of the
reconstruction.

Figure~\ref{fig:applications-ann-gist} is the GIST counterpart of
Figure~\ref{fig:applications-ann}. It shows the $13$ half-bit budgets;
the full sweep also includes the intervening quarter-bit measurements.
Lower reconstruction NMSE does not consistently improve Recall@1.

\begin{figure}[!htbp]
\centering
\includegraphics[width=\linewidth]{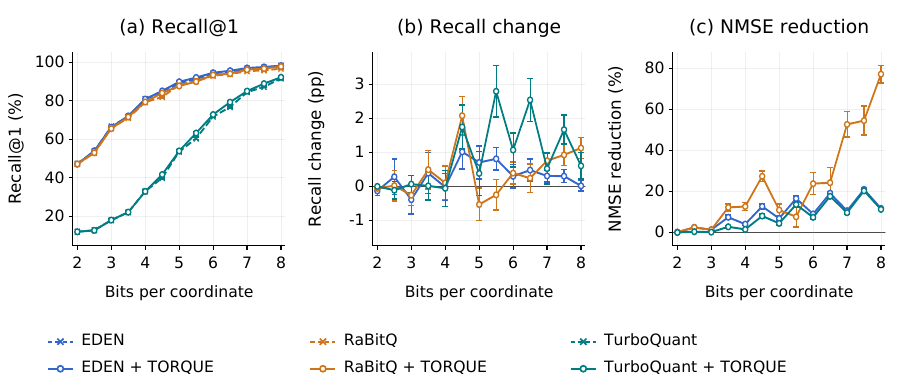}
\caption{GIST retrieval: Recall@1, recall change, and
NMSE reduction.}
\label{fig:applications-ann-gist}
\end{figure}

\paragraph{Biased reconstructions.}
Figure~\ref{fig:applications-ann-biased} evaluates EDEN-B, RaBitQ-B,
and TurboQuant-MSE on GloVe and GIST. We use the same seeds, budgets,
and retention stage as the unbiased comparisons. Combining the three
biased and three unbiased
variants gives $4{,}650$ comparisons per dataset over $775$
budget--seed configurations.

\begin{figure}[t]
\centering
\includegraphics[width=\linewidth]{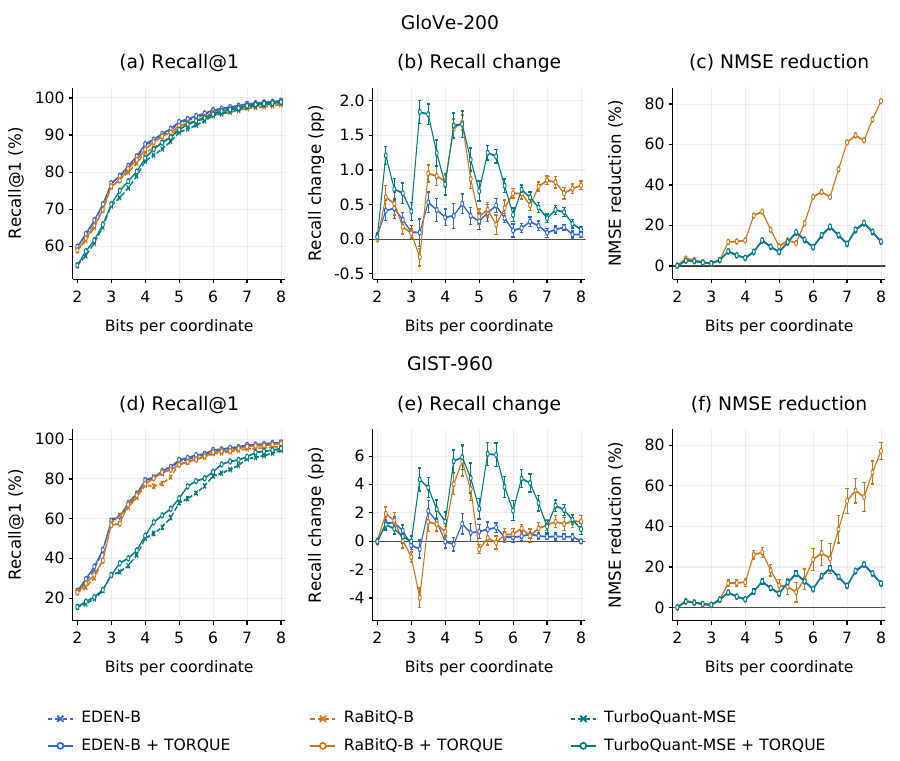}
\caption{Biased retrieval on GloVe and GIST: Recall@1,
recall change, and NMSE reduction.}
\label{fig:applications-ann-biased}
\end{figure}

\section{Native CPU and Apple GPU timing}
\label{app:local-runtime}

We extend Section~\ref{app:runtime} with runtime evaluation of \TrunQuant for CPU and Apple GPU.
All \TrunQuant\ measurements here use both retention stages.
Figures~\ref{fig:runtime-cpu} and~\ref{fig:runtime-metal} separate
parameter selection (\TrunQuant\emph{optimization}), encoding with
selected parameters (\TrunQuant\emph{encode}), and their directly timed
combination (\TrunQuant\emph{joint}). The decoding plots compare
complete reconstruction times.

\paragraph{Hardware and implementation.}
We use an Apple M2 Max with eight performance cores, four efficiency
cores, a 30-core GPU, and 32\,GiB of memory, running macOS~15.7.2.
The CPU implementation uses C++17, NEON instructions, and twelve
persistent workers; Apple Clang~15.0.0 compiles it with
\texttt{-O3 -mcpu=native -ffast-math}. The GPU uses Metal kernels
with fast math enabled. Inputs, codebooks, parameter selection,
rotations, and reconstruction use FP16 on both platforms, including
norm and correction sums. Power-of-two rescaling limits overflow
and underflow. The CPU keeps conservative rounding for its selection
bounds. Rotations use cache blocks on the CPU and shared-memory
tiles on the GPU; indices and bit packing use integer arithmetic.

\paragraph{Reducing the sorting work.}
The GPU, and the CPU for $d\ge1024$, group coordinates into $256$
buckets by magnitude.
The buckets describe how large the coordinates are without sorting them
individually. For a proposed retained count $k$, some buckets are retained
completely, some remain completely, and at most one is split. Each
bucket's range bounds the sum of squares of its remaining coordinates.
For example, ten coordinates with magnitudes between $0.10$ and $0.12$
contribute between $0.10$ and $0.144$. Combining these bounds gives an
interval for the candidate's objective $\rho_k\epsilon_d(c,s)$.
If one candidate's lower bound exceeds another's upper bound, it cannot
win and can be skipped. The remaining choices often require only a
small subset of coordinates to be sorted. When the bounds overlap,
we keep both candidates and compare them using the actual values.

\begin{figure}\centering
\includegraphics[width=0.8\linewidth]{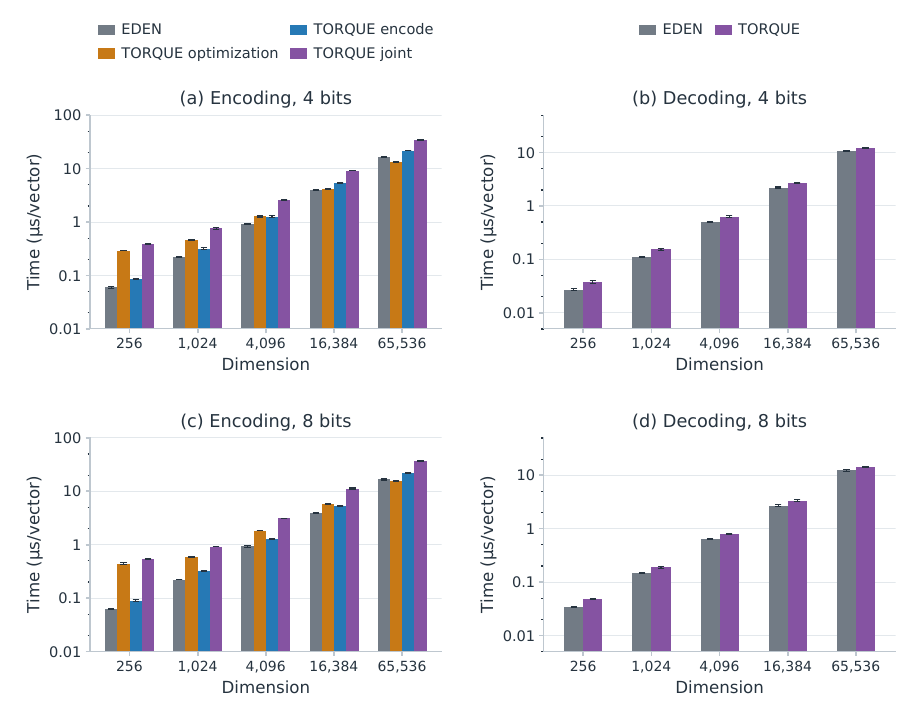}
\caption{Native CPU timing: Encoding and decoding
at four bits and eight bits.}
\label{fig:runtime-cpu}
\end{figure}

\begin{figure}
\centering
\includegraphics[width=0.8\linewidth]{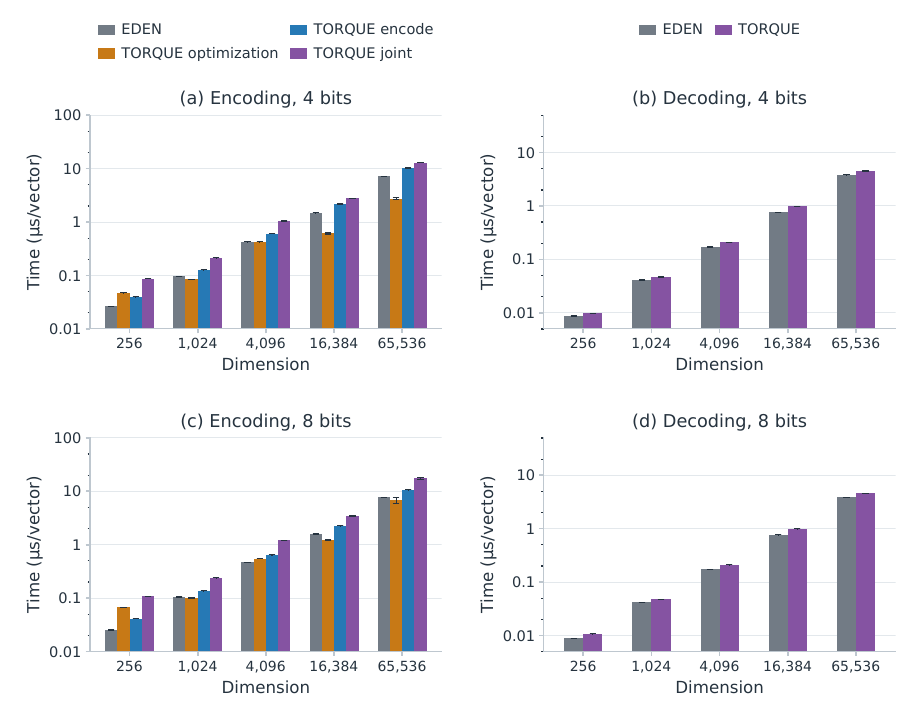}
\caption{Apple GPU timing, with the same layout as
Figure~\ref{fig:runtime-cpu}.}
\label{fig:runtime-metal}
\end{figure}

\paragraph{Inputs.}
We test $d\in\{256,1024,4096,16384,65536\}$, with $2^{23}/d$ vectors
per batch. Before timing, we generate independent Student-$t_3$
coordinates, normalize each vector in FP64, and store it in FP32.
Both implementations convert the inputs to FP16 before timing.
Seeds are $17,29,43,71,101$. NumPy~2.5.3 uses PCG64 with
\texttt{SeedSequence([seed,d,0])} for inputs and
\texttt{SeedSequence([seed,d,1])} for rotation signs. Both methods receive
the same inputs and randomized Hadamard rotations.

\paragraph{Budgets and codebooks.}
We use budgets of $b=4$ and $b=8$ expected bits per coordinate
for retained values, their indices, and quantizer codes.
For these measurements, both EDEN and \TrunQuant reserve an
additional $160$ bits per vector for metadata, excluded from $b$.
This stores norms, a reconstruction correction factor, retained
counts, and a quantizer identifier, and includes padding.

Retained values use FP16. Indices use eight bits at $d=256$
and sixteen bits at larger dimensions. Thus, for $d>256$, each
retained value and its index fit in one $32$-bit word.

We compute the codebooks offline using
\[
\begin{aligned}
S&=\{1+j/64:j=0,\ldots,448\},\\
C&=\{1.75,2,2.25,2.5,2.75,3,3.5,4,4.5,5,6,\infty\}.
\end{aligned}
\]
For each dimension and budget, we also precompute a list of feasible
parameter choices $(k,c,s)$ and their error estimates.
At runtime, we use the input to compare these choices and
select the parameters.

\paragraph{Measurements.}
We measure the time to process a batch and divide by its number
of vectors. After warm-up, we take the median runtime for each
seed and average across five seeds. Error bars show $95\%$
bootstrap confidence intervals.

CPU timings include synchronization between workers. GPU timings
cover execution on the device, with data already in GPU memory.
Each GPU timing sample processes $16$ batches and reports the average
time per batch.
Setup, host submission, and data transfers are excluded.

\paragraph{Results.}
Across the tested dimensions and budgets, the additional runtime
cost is concentrated in encoding, which includes online parameter
selection. With the parameters already selected, \TrunQuant\ encoding
takes $1.30$--$1.48\times$ EDEN's encoding time on the CPU and
$1.32$--$1.63\times$ on the GPU.
Including online selection increases these ranges to
$2.08$--$8.85\times$ and $1.82$--$4.28\times$, respectively.
The relative overhead is largest at $d=256$ and lower at every
larger dimension tested.
Decoding remains closer to EDEN throughout, taking
$1.12$--$1.41\times$ its decoding time on the CPU and
$1.13$--$1.29\times$ on the GPU.

\section{Codebook-dependent outlier retention}
\label{app:codebook-retention}

The norm threshold can be replaced by a threshold on quantization error.
Consider nearest-codeword quantization of $G\sim N(0,I_q)$ with a codebook
$Q\subset\R^q$ of $J$ points. Fix a fraction $p\in(0,1)$ of blocks to store
exactly; the remaining blocks are quantized. Let
\[
  d_Q(g)=\min_{a\in Q}\|g-a\|_2,
  \qquad \Prob\{d_Q(G)\le\tau_Q\}=1-p.
\]
For fixed $Q$, the best inlier region is
\begin{equation}
  A_Q=\{g:d_Q(g)\le\tau_Q\}
     =\bigcup_{a\in Q}B(a,\tau_Q),
  \label{eq:codebook-inlier-region}
\end{equation}
where $B(a,\tau_Q)$ is the Euclidean ball centered at $a$ with radius $\tau_Q$.
Indeed, exchanging a higher-error inlier for an equal-probability portion of
lower-error outliers reduces the error. Thus, the optimal rule quantizes the
$1-p$ fraction closest to the codebook and stores the rest exactly.

\paragraph{Joint offline optimization.}
For our objective, the Gaussian-model error per coordinate is
\[
  \epsilon(Q,p)
  =\frac1q\E\!\left[d_Q(G)^2\mathbf1_{\{G\in A_Q\}}\right]
  =\frac{\E[\min\{d_Q(G)^2,\tau_Q^2\}]-p\tau_Q^2}{q}.
\]
We can optimize $Q$ and $A_Q$ jointly by alternating two steps: (1) choose
$\tau_Q$ to give inlier probability $1-p$; (2) replace each codeword by
the Gaussian-weighted centroid of its Voronoi cell intersected with $A_Q$.
Each step does not increase the objective, although this procedure need not
find a global optimum. For each candidate codebook size and outlier fraction, we precompute a codebook, a distance threshold, and their expected \NMSE and storage cost. The existing optimizer can then select among these candidates under the bit budget. During encoding, however, each block must undergo a nearest-codeword search before the retention decision, including blocks ultimately stored at high precision. The simpler norm-based retention avoids this search.

\begin{figure}[htbp]
  \centering
  \captionsetup[subfloat]{font=small}
  \subfloat[$J=4$, $\epsilon\approx0.269643$\label{fig:codebook-retention-j4}]{%
    \includegraphics[width=0.40\linewidth]{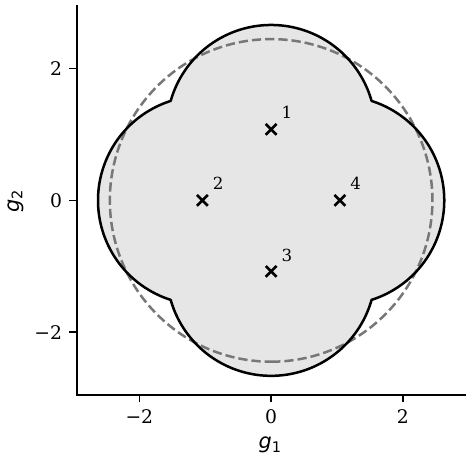}}
  \hspace{0.065\linewidth}%
  \subfloat[$J=8$, $\epsilon\approx0.143335$\label{fig:codebook-retention-j8}]{%
    \includegraphics[width=0.40\linewidth]{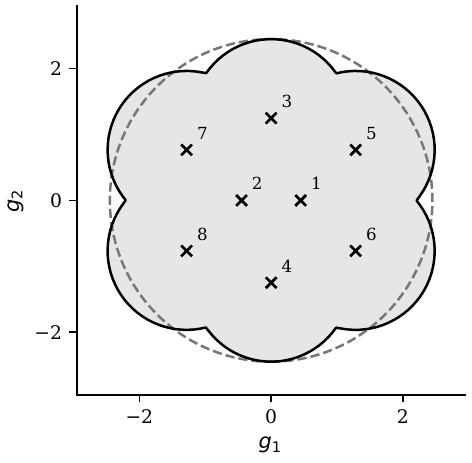}}
  \par\vspace{0.3em}
  \includegraphics[width=0.72\linewidth]{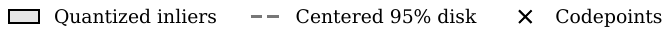}
  \caption{Codebooks numerically optimized for mean squared
  distance with $G\sim N(0,I_2)$ and $p=0.05$. The shaded union of disks and the dashed centered disk each
  contain $95\%$ of the Gaussian probability. Crosses mark all four or eight
  codewords. Blocks outside the shaded region are retained.}
  \label{fig:codebook-retention}
\end{figure}

Figure~\ref{fig:codebook-retention} shows why norm alone need not identify
the most costly blocks: some larger-norm blocks are close to a codeword,
while smaller-norm blocks between codewords can be harder to quantize.
The displayed codebooks are optimized numerically for \emph{mean squared
distance}. Their conditional mean squared distances among quantized blocks
are approximately $0.567670$ and $0.301758$ for $J=4$ and $J=8$, respectively.
Including zero error for the $5\%$ of blocks stored exactly and dividing by
$q=2$ gives Gaussian-model errors per coordinate
$\epsilon(Q,0.05)\approx0.269643$ and $0.143335$, respectively.
\par

\end{document}